\documentclass{article}

\usepackage{times}
\usepackage{graphicx} 
\usepackage{subfigure} 

\usepackage{natbib}

\usepackage{algorithm}
\usepackage{algorithmic}

\usepackage{hyperref}

\usepackage[utf8]{inputenc} 
\usepackage[T1]{fontenc}    
\usepackage{hyperref}       
\usepackage{url}            
\usepackage{booktabs}       
\usepackage{amsfonts}       
\usepackage{amsmath}
\usepackage{amsthm}
\usepackage{nicefrac}       
\usepackage{microtype}      
\usepackage{epstopdf}
\usepackage{sidecap}
\usepackage{color}
\usepackage{balance}

\newcommand{\norm}[1]{\left\lVert#1\right\rVert}

\newtheorem{theorem}{Theorem}[section]
\newtheorem{corollary}{Corollary}[theorem]

\usepackage[accepted]{icml2017}

\icmltitlerunning{LSH Partition Function Estimate}

\begin{document} 

\twocolumn[
\icmltitle{A New Unbiased and Efficient Class of LSH-Based Samplers and Estimators for Partition Function Computation in Log-Linear Models}

\begin{icmlauthorlist}
\icmlauthor{Ryan Spring}{rice}
\icmlauthor{Anshumali Shrivastava}{rice}
\end{icmlauthorlist}

\icmlaffiliation{rice}{Department of Computer Science, Rice University at Houston, TX, USA}

\icmlcorrespondingauthor{Anshumali Shrivastava}{anshumali@rice.edu}

\icmlkeywords{Partition Function, Randomized Hashing, Importance Sampling, Language Model, Deep Learning}

\vskip 0.3in
]


\printAffiliationsAndNotice{}  

\begin{abstract} 
Log-linear models are arguably the most successful class of graphical models for large-scale applications because of their simplicity and tractability. Learning and inference with these models require calculating the partition function, which is a major bottleneck and intractable for large state spaces. Importance Sampling (IS) and MCMC-based approaches are lucrative. However, the condition of having a "good" proposal distribution is often not satisfied in practice. 

In this paper, we add a new dimension to efficient estimation via sampling. We propose a new sampling scheme and an unbiased estimator that estimates the partition function accurately in sub-linear time. Our samples are generated in near-constant time using locality sensitive hashing (LSH), and so are correlated and unnormalized. We demonstrate the effectiveness of our proposed approach by comparing the accuracy and speed of estimating the partition function against other state-of-the-art estimation techniques including IS and the efficient variant of Gumbel-Max sampling. With our efficient sampling scheme, we accurately train real-world language models using only 1-2\% of computations.
\end{abstract} 

\section{Introduction}
\label{sec:introduction}
Probabilistic graphical models are some of the most flexible modeling frameworks in machine learning, physics, and statistics. A common and convenient way of modeling probabilities is only to model the proportionality function. Such a specification is sufficient because proportionality can be uniquely converted into actual probability value by dividing them by {\em normalization constant}. The normalization constant is more popularly known as the {\em partition function}.

Log-linear models~\cite{koller2009probabilistic, lauritzen1996graphical} are arguably the most successful class of graphical models for large-scale applications. These models include multinomial logistic (Softmax) regression and conditional random fields. Even the famous skip-gram models \cite{goodman2001bit, mikolov2013distributed} are examples of log-linear models. The definition of a log-linear model states that the logarithm of the model is a linear combination of a set of features $x \in \mathcal{R}^D$. Assume there is a set of states $Y$. Each state $y \in Y$ is represented with a weight vector $\theta$. A frequent task is estimating the probability of a state $y \in Y$. The probability distribution for a log-linear model is $$ P(y|x,\theta) = \frac{e^{\theta_y \cdot x}}{Z_\theta} $$ where $\theta_y$ is the  weight vector, $x$ is the (current context) feature vector, and $Z_\theta$ is the partition function. The partition function is the normalization constant the ensures that $P(y|x, \theta)$ is a valid probability distribution. $$Z_\theta = \sum_{y \in Y} e^{\theta_y \cdot x}$$ Computing the partition function requires summing over all of the states $Y$. In practice, it is an expensive, intractable operation when the size of the state space is enormous.

The value of partition function is required during training and inference.  Assume there is a training set containing $N$ labeled examples $[(x_1, y_1), \dots, (x_N, y_N)]$. The model is trained by minimizing the negative log-likelihood using stochastic gradient descent (SGD).
$$L(\theta) = -\frac{1}{N} \sum_{x \in X} \theta \cdot x + \log(Z_\theta)$$
$$\nabla L(\theta) = -\frac{1}{N} \sum_{x \in X} 1[y_i = k] - P(y_i = k | x_i; \theta)$$
Here, computing $P(y_i = k | x_i; \theta)$ requires the value of partition function $Z_\theta$.  This process is near-infeasible when the size of $|Y|$ is huge. It is common to have scenarios in NLP~\cite{chelba2013one} and vision~\cite{deng2009imagenet} with the size of the state space running  into millions.

Due to the popularity of log-linear models, reducing the associated computational challenges has been one of the well-studied and emerging topics in large scale machine learning literature.  For convenience, we classify existing line of work concerning efficient log-linear models into three broad categories: 1) Classical Sampling or Monte Carlo Based, 2) Estimation via Gumbel-Max Trick~\cite{mussmann2016learning}, and 3) Heuristic-Based.

{\bf 1. Classical Sampling or Monte Carlo:} Since the partition function is a summation, it can be very well approximated in a provably unbiased fashion using Monte Carlo or Importance sampling (IS). IS and its variants, such as annealed importance sampling (AIS) \cite{neal2001annealed}, are probably the most widely used Monte Carlo methods for estimating the partition function in general graphical model. IS works by drawing samples $y$ from a tractable proposal (or reference) distribution $y \sim g(y)$, and estimates the target partition function $Z_\theta$ by averaging the importance weights $f(y)/g(y)$ across the samples [see Section~\ref{sec:IS}], where $f(y) = e^{\theta_y \cdot x}$ is the unnormalized target density.

It is widely known the IS estimate often has very high variance, if the choice of proposal distribution is very different from the target, especially when they are peaked differently. In fact, there is no known effective class of proposal distribution in literature for log-linear models. This line of work is considered a dead end because sampling from a good proposal is almost as hard as sampling from the target. Our solution changes this belief and shows a provable, efficient proposal distribution (unnormalized) and a corresponding unbiased estimator for partition function whose computational complexity is amortized sub-linear time.

{\bf 2. Gumbel-Max Trick:} Previous work \cite{gumbel1954statistical} has shown an elegant connection between the partition function of log-linear models and the maximum of a sequence of numbers perturbed by Gumbel distribution (see Section~\ref{sec:Gumbel}). The bottom line is the value of the log partition function is estimated by finding the maximum value of the state space $Y$ perturbed by Gumbel noise. This observation does not directly lead to any computational gains in partition function estimation because computing the maximum still requires enumerating over the entire state space. Very recently, \cite{mussmann2016learning} showed that computing the same maximum can be reformulated as a maximum inner product search (MIPS) problem, which can be approximately solved efficiently using recent algorithmic advances~\cite{shrivastava2014asymmetric, shrivastava2014improved,shrivastava2015asymmetric}. The overall method requires a single costly pre-processing phase. The initial cost of the pre-processing phase is amortized over several fast inference queries. Since the cost of approximate MIPS is much smaller, estimating the partition function is more efficient than the brute force Gumbel-Max Trick.

Unfortunately, as we show in this paper (Section~\ref{sec:Gumbel_Flaws}), even small perturbations in the identity of the maximum leads to significant performance deviation and poor accuracy. An important thing to note is that since this method needs a MIPS (approximate nearest-neighbor) query for generating a single sample, it is quite inefficient. Although MIPS and other near-neighbor queries are sub-linear time operations, they still have a significant cost. In theory, the time complexity is $N^\rho$ where $\rho < 1$. Moreover, the accuracy is very sensitive to the approximation of the maximum value [see Section~\ref{sec:Gumbel_Flaws} for details]. We empirically demonstrate that the cost of estimating partition function using a MIPS query per sample is not only inaccurate but also prohibitively slow.

{\bf 3. Heuristic-Based:}  There are other approaches that avoid estimating the partition function completely. Instead, they approximate the original log-linear model with an altogether different model, which is cheaper to train. The most popular is the Hierarchical Softmax~\cite{morin2005hierarchical}. Changing the model's assumption based on some heuristic hierarchy may not be desirable in many application, and its effect on the accuracy is not very well understood. It is further known that such models are sensitive to the structure of the Hierarchical Softmax approach \cite{mikolov2013distributed}

{\bf Focus of this paper:} The focus of this article is on an efficient partition function estimation in log-linear models without any additional assumptions. We will focus on simple, efficient, and unbiased estimators with superior properties. We want to retain the original modeling assumption, and therefore, we will not focus on techniques like hierarchical softmax, which has an explicit hierarchical assumption. All these assumptions are unreliable. In light of our goals, we will ignore (3) Heuristic based techniques and only focus on (1) Classical Sampling or Monte Carlo Based and (2) Estimation via Gumbel-Max Trick as our baselines.

{\bf Our Contributions:} Our proposal also exploits the MIPS (Maximum Inner Product Search) data structure. The key difference is that the existing approaches \cite{mussmann2016learning} require a MIPS query (relatively costly) to generate each informative sample. On the other hand, our approach can generate a large set of samples from an elegant, informative proposal distribution using a single MIPS query! We explain this difference in Section~\ref{sec:methodology}.

Our work is in fact completely different from all existing works. Instead of relying on the Gumbel-Max Trick, we return to the basics of sampling and estimation. We reveal a very unusual, but super-efficient class of samplers, which produces a set of correlated samples that are not normalized. We further show an unbiased estimator of partition function using these unusual samples. Our proposal opens a new dimension for sampling and unbiased estimation beyond classical IS, which is worth of study in its own right. The estimators are generic for any partition function estimation.

For log-linear models, we show that our sampling scheme $P_{MIPS}(y)$ has many similar properties of the target distribution, such as same modes and same peaks, making it very informative for estimation. Most interestingly, it is possible to generate $T$ samples from the magical distribution $P_{MIPS}(y)$ in sub-linear time. To the best of our knowledge, this is also the first work that constructs a provably efficient and informative sampling distribution for log-linear models. 

We show that our LSH sampler provides the perfect balance between speed and accuracy when compared to the other approaches - Uniform IS, Exact Gumbel, and MIPS Gumbel. Our LSH method is more accurate at estimating the partition function than the Uniform IS method while being equally fast. In addition, our method is several orders of magnitude faster than the Exact Gumbel and MIPS Gumbel techniques while maintaining competitive accuracy. Furthermore, our method successfully trains real-world language models accurately while  only requiring 1-2\% of the states to estimate the partition function.

\section{Background}
\label{sec:background}
\subsection{Importance Sampling}
\label{sec:IS}
Assume we have a proposal distribution $g(y)$ where $\int g(y) dy = 1$. Using the proposal distribution, we obtain an unbiased estimator of the partition function $Z_\theta$.

\begin{center}
$\mathbb{E} \Big [\frac{f(y)}{g(y)} \Big] = \sum_y g(y) \frac{f(y)}{g(y)} = \sum_y f(y) = Z_\theta$
\end{center}

We draw $N$ samples from the proposal distribution $y_i \sim g(y)$ for $i = 1 \dots N$. Using those samples, we have a Monte-Carlo approximation of the partition function $Z_\theta$.

\begin{center}
$ Z^{-1} = \mathbb{E} \big [ e^{-H} \big ] $
\end{center}

\subsection{Locality Sensitive Hashing (LSH)}
\label{sec:LSH}
Locality-Sensitive Hashing (LSH) \cite{gionis1999similarity, huang2015query, gao2014dsh, shinde2010similarity} is a popular, sub-linear time algorithm for approximate nearest-neighbor search. The high-level idea is to place similar items into the same bucket of a hash table with high probability. An LSH hash function maps an input data vector to an integer key - $h(x): \mathbb{R}^D \mapsto [0, 1, 2,\dots, N]$. A collision occurs when the hash values for two data vectors are equal - $h(x) = h(y)$. The collision probability of most LSH hash functions is generally a monotonic function of the similarity - $Pr[h(x) = h(y)] = \mathcal{M}(sim(x,y))$, where $\mathcal{M}$ is a monotonically increasing function. Essentially, similar items are more likely to collide with each other under the same hash fingerprint. 

\begin{figure} [ht]
\begin{center}
  \includegraphics[width=0.45\textwidth]{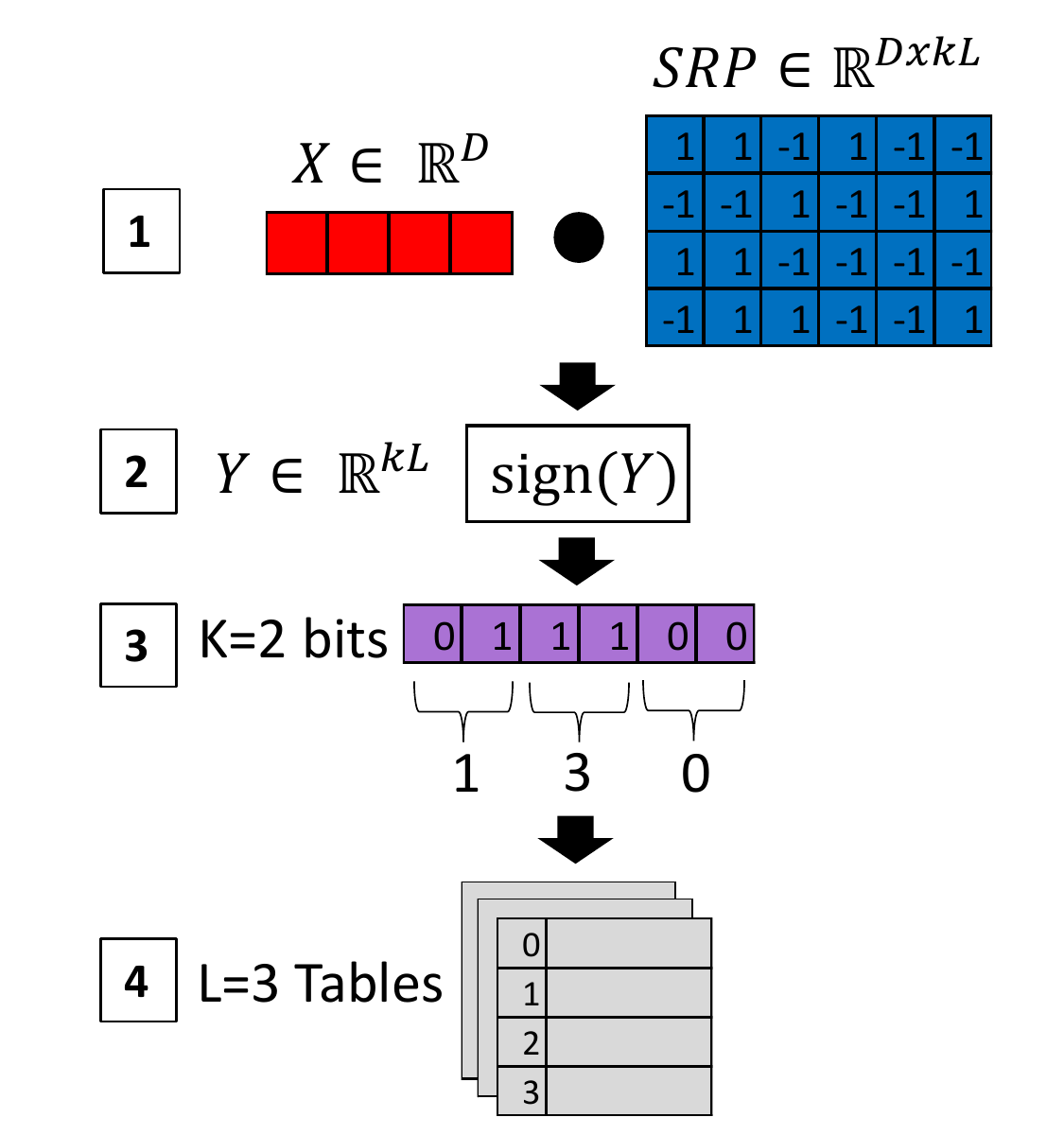}
\end{center}
\vspace{-0.2in}
\caption{
{\bf Locality Sensitive Hashing - Signed Random Projections}
{\bf (1)} Compute the projection using a signed, random matrix $\mathcal{R}^{D \times kL}$ and the item $x \in \mathcal{R}^D$.
{\bf (2)} Generate a bit from the sign of each entry in the projection $\mathcal{R}^{kL}$
{\bf (3)} From the $kL$ bits, we create $L$ integer fingerprints with $k$ bits per fingerprint.
({\bf 4)} Add the item $x$ into each hash table using the corresponding integer key
}
\label{fig:lsh_pipeline}
\end{figure}

The algorithm uses two parameters - $(K, L)$. We construct $L$ independent hash tables from the collection $\mathcal{C}$. Each hash table has a meta-hash function $H$ that is formed by concatenating $K$ random independent hash functions from $\mathcal{F}$. Given a query, we collect one buckets from each hash table and return the union of $L$ buckets. Intuitively, the meta-hash function makes the buckets sparse (less crowded) and reduces the amount of false positives because only valid nearest-neighbor items are likely to match all $K$ hash values for a given query. The union of the $L$ buckets decreases the number of false negatives by increasing the number of potential buckets that could hold valid nearest-neighbor items.

The candidate generation algorithm works in two phases (See~\cite{Report:E2LSH} for details):
\begin{enumerate}
    \itemsep0em
    \item {\bf Pre-processing Phase:} We construct $L$ hash tables from the data by storing all elements $x \in \mathcal{C}$. We only store pointers to the vector in the hash tables because storing whole data vectors is very memory inefficient.
    \item {\bf Query Phase:} Given a query $Q$, we will search for its nearest-neighbors. We report the union from all of the buckets collected from the $L$ hash tables. Note, we do not scan all the elements in $\mathcal{C}$, we only probe $L$ different buckets, one bucket for each hash table. 
\end{enumerate}
After generating the set of potential candidates, the nearest-neighbor is computed by comparing the distance between each item in the candidate set and the query.

\subsection{SimHash}
\label{sec:SimHash}
SimHash or Signed Random Projection (SRP)~\cite{charikar2002similarity} is the LSH family for the Cosine Similarity metric. The Cosine Similarity metric is the angle between two vectors $x,y \in \mathcal{R}^D$. Since the definition of an inner product is $X \cdot Y = \Sigma_{i=0}^{N} x_i \cdot y_i = \norm{X}\norm{Y}cos(\theta)$, the simple formula for the angle between two vectors is $\theta = cos^{-1}(\frac{x \cdot y}{\norm{x}\norm{y}})$. For a vector $x$, the SRP function generates a random hyperplane $w$ and returns the sign of the projection of $x$ onto $w$. Two vectors share the same sign only if the random projection does not fall in-between them. Since all angles are equally likely for a random projection, the probability that two vectors x, y share the same sign for a given random projection is $1-\frac{\theta}{\pi}$. Using the signed random projection hash function, we can create an $(\theta_1,\theta_2,1- \frac{\theta_1}{\pi},1-\frac{\theta_2}{\pi})$-sensitive LSH family.

\section{MIPS Reduction using the Gumbel Distribution}
\subsection{Gumbel Distribution}
\label{sec:Gumbel}
 The Gumbel distribution $\mathcal{G}$ \cite{gumbel1941} is a continuous probability distribution with the cumulative distribution function (CDF) - $P(G \leq x) = e^{-e^{-\frac{x - \mu}{\beta}}}$


A key technique that uses the Gumbel distribution is the Gumbel-Max Trick. \cite{gumbel1954statistical}

\begin{center}
$H = \underset{y \in Y}{\max} [\phi_y + G(y)] \sim \log{(\sum_{y \in Y} e^{\phi_y})} + G$ 
\end{center}

where $\phi_y = \theta_y x$ is the log probability of the log-linear model, $G(y)$ is an independent Gumbel random variable for each state $y$, and $G$ is an independent Gumbel random variable. Using the Gumbel-Max Trick, we can estimate the inverse partition function $Z_\theta$. 

\begin{center}
$ Z^{-1} = \mathbb{E} \big [ e^{-H} \big ]$

$ \hat{Z}^{-1} = \frac{1}{N} \sum_{i}^{N} \big [ e^{-H_i} \big ] $
\end{center}

\cite{mussmann2016learning} proposed using two algorithms, Maximum Inner Product Search (MIPS) and Gumbel-Max Trick, to estimate the partition function $Z_\theta$ efficiently. Their idea was to convert the Gumbel-Max Trick into a Maximum Inner Product Search (MIPS) problem. We will provide a brief overview of their approach.

\subsection{Algorithm}
For their MIPS reduction, the first step is to build the MIPS data structure. A vector of $k$ independent Gumbel random variables is concatenated to each weight vector $\theta_y$ to form the Gumbel-weight vector $v_y = (\theta_y, \{G_{y,j}\}_{j=1}^{k})$.
These Gumbel-weight vectors $v_y$ are added to the MIPS data structure. During the partition estimate phase, a subset of new weight vectors is queried from the MIPS data structure. A one-hot vector $e_j$ where $j = 1 \dots k$ is concatenated to the feature vector $x$ to form the query $q_j = (x, e_j)$.
The one-hot vector ensures that only the $j^{th}$ Gumbel variable is selected at one time. The MIPS data structure returns a subset of Gumbel-weight vectors $S_j$ that is likely to produce the maximum inner product with the query $q_j$. The exact maximum inner product is computed for this small subset of Gumbel-weight vectors $S_j$ and the query $q_j$. The overall process is still sub-linear time, reducing the search space for the exact Gumbel-Max trick and improving the performance of the partition function estimation process. The efficiency comes at the cost of approximate answers to MIPS queries.

\begin{figure}[t]
\begin{center}
  \includegraphics[width=0.45\textwidth]{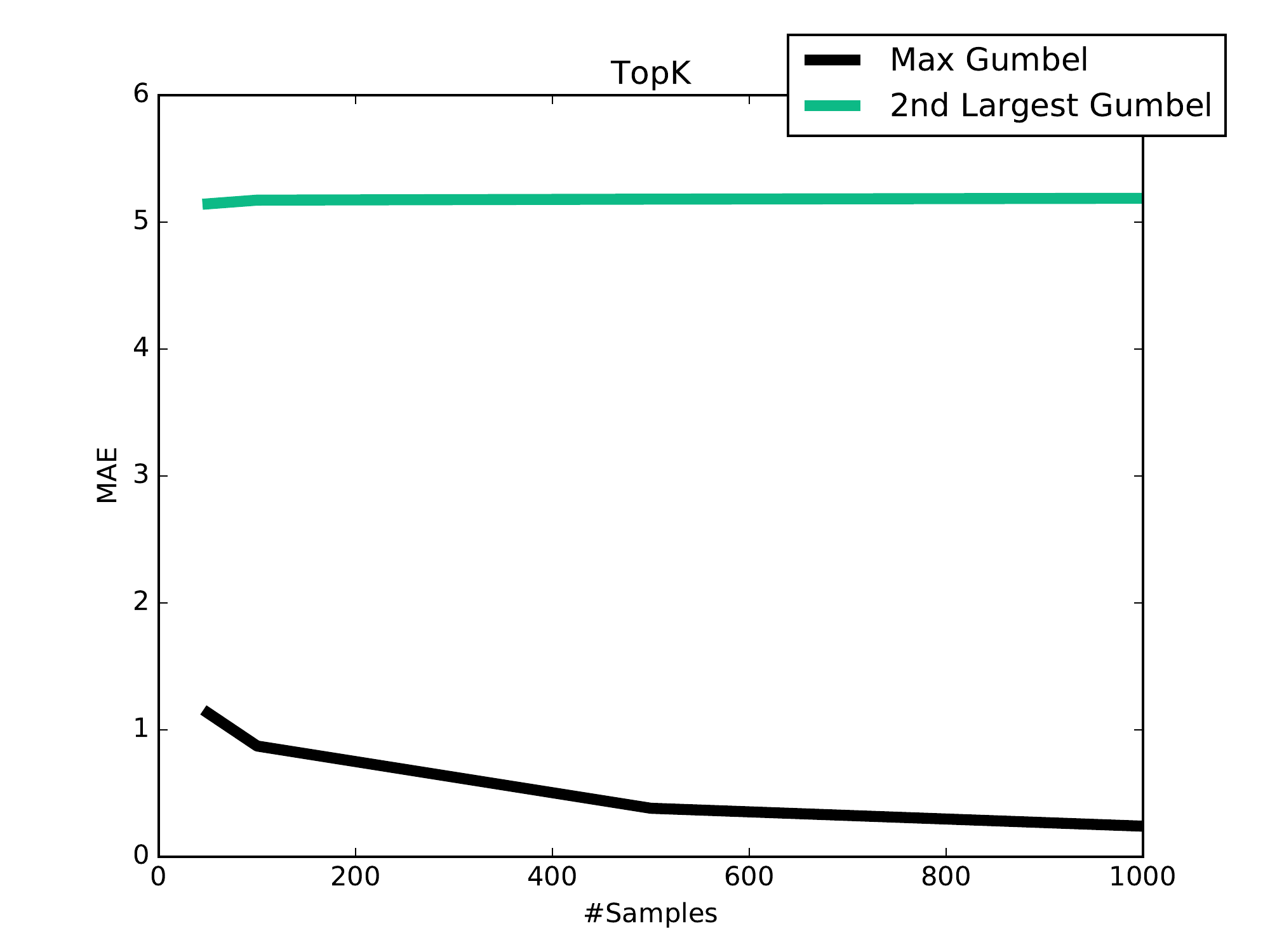}
\end{center}
\vspace{-0.25in}
\caption{
{\bf (1)} Exact Max Gumbel - Estimate the partition function using the maximum value over all states
{\bf (2)} 2nd Largest Gumbel - The maximum (top-1) value is replaced with the 2nd largest (top-2) value in the partition function estimate.
Notice, there is a large gap in accuracy when the 2nd largest value is used to estimate the partition function. Therefore, using the maximum value for the estimate is a requirement for an accurate estimate.
}
\label{fig:topk}
\end{figure}

\subsection{MIPS-Gumbel Reduction is Inaccurate and Inefficient}
\label{sec:Gumbel_Flaws}
The MIPS-Gumbel Reduction has two main weaknesses in terms of speed and accuracy. The MIPS data structure obtains an approximate nearest-neighbor set for a query. Due to the randomness in the MIPS data structure, there is a chance that the set may not contain the exact maximum value. In Figure \ref{fig:topk}, we empirically explore the accuracy of the partition function estimate when the $2^{nd}$ largest value is substituted for the exact maximum value. It shows that missing the exact maximum value drastically increases the error of the partition function estimate. Since the estimates are extremely sensitive to perturbations in the maximum value, it is necessary to have a high confidence MIPS algorithm. It is well know that high confidence search is inefficient, requiring large number of hash functions and tables.

In addition, there is another issue that affects the performance of the MIPS-Gumbel Reduction. Each sample requires querying the MIPS data structure for a subset $S_j$ and then taking the maximum inner product between the subset and the query. This cost of approximate MIPS query, although sub-linear, is still a significant fraction of the $N$.  Furthermore, we need large enough $T$ for accurate estimation, which amounts to $T$ MIPS queries, which is likely to be inefficient. Our evaluations clearly validates the inefficiency of this approach on large datasets.

\section{Key Observation: LSH is an Efficient Informative Sampler in Disguise}
The traditional LSH algorithm retrieves a subset of potential candidates for a given query in sub-linear time. We compute the actual distances of these neighbors for this candidate subset and then report the closest nearest-neighbor. A close observation reveals that an item returned as candidate from a $(K,L)$ parametrized LSH algorithm is sampled with probability $1 - (1 - p^K)^L$ where $p$ is the collision probability of LSH function. 

For the classic LSH algorithm, the probability of retrieving any item $y$ for a given query context $x$ can be computed exactly as follows \cite{leskovec2014mining}:
\begin{enumerate}
    \itemsep-0.3em
    \item The probability that the hash fingerprints match for a random LSH function - $Pr[h(x) = h(y)] = p$
    \item The probability that the hash fingerprints match for a meta-LSH function - $Pr[H(x) = H(y)] = p^k$
    \item The probability that there is at least one mismatch between the $K$ hash fingerprints that compose the meta-LSH function - $Pr[H(x) \neq H(q)] = 1 - p^k$
    \item The probability that none of the $L$ meta-hash fingerprints match - $Pr[H(x) \neq H(y)] = (1 - p^k)^L$
    \item The probability that at least one of the $L$ meta-hash fingerprints match and the two items are a candidate pair - $Pr[H(x) = H(y)] = 1 - (1 - p^k)^L$
\end{enumerate}
The precise form of $p$ is defined by the LSH family used to build the hash tables. We can construct a MIPS hashing scheme such that $p = \mathcal{M}(q \cdot x) =  \mathcal{M}(\theta_y \cdot x)$ where $\mathcal{M}$ is a monotonically increasing function. 

However, the traditional LSH algorithm does not represent a valid probability distribution $\sum_{i=1}^N Pr(y_i) \neq 1$. Also, due to the nature of LSH, the sampled candidates are likely to be very correlated. Thus, standard techniques like IS are not applicable to this kind of samples. It turns out that there is a simple, unbiased estimator for the partition function using the samples from the LSH algorithm. We take a detour to define a general class of sampling and partition function estimators where the LSH sampling is a special case.

\section{A New Class of Estimators for Partition Function}
Assume there is a set of states $Y = [y_1 \dots y_N]$. We associate a probability value with each state $[p_1 \dots p_N]$.
Define the sampling process as follows: 

We select each of these states $y_i$ to be a part of the sample set $S$ with probability $p_i$. Note, the probabilities need not sum to 1, and the sampling process is allowed to be correlated. Thus, we get a correlated sample set $S$. It can be seen that MIPS sampling is a special class of this sampling process with $p_i = 1 - (1 - p^k)^L$. 

Given the sample set $S$, we have an unbiased estimator for any partition function $\sum_{y_i \in Y} f(y_i)$.
\begin{theorem}
\label{lsh_sampler}
Assume that every state $y_i$ has a weight given by $f(y_i)$ with partition function $\sum_{y_i \in Y} f(y) = Z_\theta$. Then we have the following as an unbiased estimator of $Z_\theta$:
\begin{align}
  Est &= \sum_{y_i \in S} \frac{f(y_i)}{p_i} = \sum_{i =1}^N {\bf 1}_{y_i \in S} \cdot \frac{f(y_i)}{p_i}\\
\mathbb{E}[Est] &= \sum_{i=1}^{N} f(y_i) = Z_\theta
\end{align}
\end{theorem}

\begin{theorem}
The variance of the partition function estimator is:
\begin{align}
Var[Est] &= \sum_{i=1}^{N} \frac{f(y_i)^2}{p_i} - \sum_{i=1}^{N} f(y_i)^2 \\
& + \sum_{i\ne j} \frac{f(y_i)f(y_j)}{p_ip_j} \mathrm{Cov}({\bf 1}_{[y_i \in S]} \cdot {\bf 1}_{[y_j \in S]})
\end{align}
If the states are selected independently, then we can write the variance as:
$$ Var[Est] = \sum_{i=1}^{N} \frac{f(y_i)^2}{p_i} - \sum_{i=1}^{N} f(y_i)^2 $$
\end{theorem}

{\bf Note 1:} In general, this sampling process is inefficient. We need to flip coins for every state in order to generate the sample set $S$.
For log-linear models with feature vector $x$ and function $f(y_i) = e^{\theta_{y_i} \cdot x}$, we show a particular form of probability $p_i = 1 - (1 - \mathcal{M}(\theta_{y_i} \cdot x)^k)^L$) for which this sampling scheme is very efficient. In particular, we can efficiently sample for a sequence of queries with varying $x$ in amortized near-constant time [See Section~\ref{sec:LSH}].

{\bf Note 2:} In our case, where these probabilities $p_i$ are generated from LSH (or ALSH for MIPS), the term $\sum_{i\ne j} \frac{f(y_i)f(y_j)}{p_ip_j} \mathrm{Cov}({\bf 1}_{[y_i \in S]} \cdot {\bf 1}_{[y_j \in S]})$ contains very large negative terms. For each dissimilar pair $y_i, y_j$, the term $\mathrm{Cov}({\bf 1}_{[y_i \in S]} \cdot {\bf 1}_{[y_j \in S]})$ is negative. When ${\bf 1}_{[y_i \in S]} = 1$ and ${\bf 1}_{[y_j \in S]} = 1$, it implies that $y_i$ and $y_j$ are both similar to the query. Therefore, they are similar to each other due to triangle inequality~\cite{charikar2002similarity}. Thus, for random pairs $y_i, y_j$, the covariance will be negative. i.e. If $y_i$ is sampled, then $y_j$ has less chance of being sampled and vice versa. Hence, we can expect the overall variance with LSH-based sampling to be significantly lower than uncorrelated sampling. This is something unique about LSH, and so it is super-efficient and its correlations are beneficial.

\subsection{Why is MIPS the correct LSH function for Log-Linear Models?}
The terms $\sum_{i=1}^{N} \frac{f(y_i)^2}{p_i} $ in the variance is similar in nature to the $\chi^2(f||p)$ term in the variance of Importance Sampling (IS) \cite{liu2015estimating}. The variance of the IS estimate is high when the target $f$ and the proposal $p$ distributions are peaked differently. i.e. they give high mass to different parts of the sample space or have different modes Therefore, for similar reasons as importance sampling, our scheme is likely to have low variance when $f$ and $p_i$ are aligned. It should be noted that there are very specific forms of probability $p_i$ for which the sampling is efficient. We show that with the MIPS LSH function, the probabilities $p_i$ and the function $f(y_i) = e^{\theta_{y_i} \cdot x}$ align well.

We have the following relationship between the probability of each state $p_i$ and the Log-Linear unnormalized target distribution $P(y|x,\theta)$.
\begin{theorem}
\label{eq:sample_prob}
  For any two states $y_1$ and $y_2$:
  $$ P(y_1|x;\theta) \ge P(y_2|x;\theta) \iff p_1 \ge p_2$$ 
  where 
  $$p_i = 1 - (1 - \mathcal{M}(\theta_{y_i} \cdot x)^K)^L$$ 
  $$P(y|x,\theta) \propto e^{\theta_y \cdot x}$$
\end{theorem}

\begin{corollary}
The modes of both the sample and the target distributions are identical.
\end{corollary}
Therefore, we can expect hashing for MIPS to be a good choice for low variance.

\section{Estimate Partition Function using LSH Sampling}
\label{sec:methodology}
The combination of these observations is a fast, scalable approach for estimating the partition function of Log-linear models. The pseudo-code for this process is shown in Algorithms \ref{alg:initialize} and \ref{alg:estimate}.

Here is an overview of our LSH sampling process:  
\begin{enumerate}
    \itemsep-0.3em
    \item During the pre-processing phase, we use randomized hash functions to build hash tables from the weight vectors $\theta_y$ for each state $y \in Y$.
    \item For each partition function estimate, we sample weight vectors from the hash tables with probability proportional to the unnormalized density of the weight vector for the state and the feature vector $e^{\theta_y \cdot x}$.
    \item For each weight vector $\theta_y$ in the sample set $S$, we determine $p$, the probability of a hash collision with the feature vector $x$. This probability is dependent on the LSH family used to build the hash tables. We chose an LSH family such that the probability $p$ is monotonic with respect to the inner product $\theta_y \cdot x$
    \item The partition function estimate for the feature vector $x$ is the sum of each weight vector $\theta_y$ in the sample set $S$ weighted by the hash collision probability $p$. 
    
    \begin{center}
    $ \hat{Z_\theta} = \sum_{i=1}^{N} 1_{[y_i \in S]} \cdot \frac{f(y_i)}{p_i} $
    \end{center}
\end{enumerate}

{\bf Running Time:} There is a key distinction in performance between the MIPS-Gumbel Reduction and our LSH sampler. The MIPS-Gumbel Reduction needs to query the MIPS data structure for each individual sample. Our LSH sampler uses a single query to the LSH data structure to retrieve the entire sample set $S$ for the partition function estimate.  Our sampling is roughly constant time. 

In terms of complexity, the MIPS-Gumbel Reduction requires $T$ full nearest-neighbor queries, which includes the costly filtering of retrieved candidates.  While the running time for our LSH Sampling estimate is the cost of a single LSH query for all the samples. In section \ref{sec:accuracy_speed}, we support our complexity analysis with empirical results, which our LSH Sampling estimate is significantly faster than the MIPS-Gumbel Reduction.

\begin{algorithm}[tb]
   \caption{LSH Sampling - Initialization}
   \label{alg:initialize}
\begin{algorithmic}
   \STATE 
   \STATE {\bf Input:} $[\theta_y]_{y \in Y}$ weight vectors, k, L
   \STATE HT = Create(k, L)
   \FOR {\textbf{each} $y \in Y$}
   \STATE Insert(HT, $\theta_y$)
   \ENDFOR
   \STATE {\bf Return:} HT
\end{algorithmic}
\end{algorithm}

\begin{algorithm}[tb]
   \caption{LSH Sampling - Partition Estimate}
   \label{alg:estimate}
\begin{algorithmic}
   \STATE {\bf Input:} 
   \STATE LSH data structure HT, k, L
   \STATE $[\theta_y]_{y \in Y}$ weight vectors
   \STATE $x$ feature vector
   \STATE $p(x,y)$ --- LSH Collision Probability
   \\\hrulefill
   \STATE union = query(HT, x)
   \STATE total = 0
   \FOR {{\bf each} $y \in union$}
   \STATE $weight = 1 - (1 - p(x,y)^k)^L$
   \STATE $logit = e^{\theta_y \cdot x}$
   \STATE $total \mathrel{+}= \frac{logit}{weight}$
   \ENDFOR
   \STATE {\bf Return:} $\hat{Z_\theta} = total$
   \STATE
\end{algorithmic}
\end{algorithm}

\section{Experiments}
\label{sec:experiments}
We design experiments to answer the following four important questions: 
\begin{enumerate}
    \itemsep-0.3em
    \item How accurately does our LSH Sampling approach estimate the partition function?
    \item What is the running time of our Sampling approach?
    \item How does our Sampling approach compare with the alternative approaches in terms of speed and accuracy?
    \item How does using our LSH Sampling approach affect the accuracy (perplexity) of real-world language models?
\end{enumerate}

For evaluation, we implemented the following three approaches to compare and contrast against our approach.
\begin{itemize}
    \itemsep-0.3em
    \item Uniform Importance Sampling: An IS estimate where the proposal distribution is a uniform distribution U[0, N]. All samples are weighted equally.
    \item Exact Gumbel: The Max-Gumbel Trick is used to estimate the partition function. The maximum over all of the states is used for this estimate.
    \item MIPS Gumbel \cite{mussmann2016learning}: A MIPS data structure is used to collect a subset of the states efficiently. This subset contains the states that are most likely to have a large inner product with the query. The Max-Gumbel Trick estimates the partition function using the subset instead of all the states.
\end{itemize}

\subsection{Datasets}
\begin{itemize}
    \itemsep-0.3em
    \item Penn Tree Bank (PTB) \cite{marcus1993building} \footnote{\url{http://www.fit.vutbr.cz/˜imikolov/rnnlm/simple-examples.tgz}} - This dataset contains a vocabulary of 10K words. It is split into 929k training words, 73k validation words, and 82k test words.
    \item Text8 \cite{mikolov2014learning} \footnote{\url{http://mattmahoney.net/dc/text8.zip}}- This dataset is a preprocessed version of the first 100 million characters from Wikipedia. It is split into a training set (first 99M characters) and a test set (last 1M characters)  It has a vocabulary of 44k words.
\end{itemize}

\subsection{Training Language Models}
The goal of a neural network language model is to predict the next word in the text given the previous history of words. The performance of language models is measured by its perplexity. The perplexity score $e^{loss}$ measures how well the language model is likely to predict a word from a dataset. The loss function is the average negative log likelihood for the target words. 

$$loss = -\frac{1}{N} \sum_{i=1}^{N} \log{p_{y_i}}$$

For our experiments, our language model is a single layer LSTM with 512 hidden units. The size of the input word embeddings is equal to the number of hidden units. The model is unrolled for 20 steps for back-propagation through time (BPTT). We use the Adagrad optimizer with an initial learning rate of 0.1 and an epsilon parameter of 1e-5. We also clip the norm of the gradients to 1. The models are trained for 10 epochs with a mini-batch size of 32 examples. The output layer is a softmax classifier that predicts the next word in the text using the context vector $x$ generated by the LSTM. The entire vocabulary for the text is the state space $Y$ for the softmax classifier.

In this experiment, we test how well the various approaches estimate the partition function by training a language model. At test time, we measure the effectiveness of each approach by using the original partition function and comparing the model's perplexity scores. The settings for our LSH data structure were k=10 bits and L=16 tables. Using these settings, our approach samples around 1.5-2\% of the entire vocabulary for its partition function estimate. (i.e. 200 samples - PTB, 800 samples - Text8) The Uniform IS estimate uses the same number of samples as our LSH estimate. For each Exact Gumbel estimate, we randomly sample 50 out of 1000 Gumbel random variables. For the MIPS Gumbel approach, we use 50 samples per estimate and an LSH data structure with k=5 bits and L=16 tables that collects around 35\% of the entire vocabulary per sample.

From Table \ref{fig:LM}, our approach closely matches the accuracy of the standard partition function with minimal error. In addition, the poor estimate from the Uniform IS approach results in terrible performance for the language model. This highlights the fact that an accurate, stable estimate of the partition function is necessary for successfully training of the log-linear model. The Exact Gumbel approach is the most accurate approach but is significantly slower than our LSH approach. The MIPS Gumbel approach diverged during training because of its poor accuracy in estimating the partition function.


\begin{figure} [ht]
\begin{center}
    \begin{tabular}{ |c|c|c| p{1.25cm} | p{1.25cm}| } 
    \hline
    Standard & LSH & Uniform & Exact Gumbel & MIPS Gumbel \\
    \hline
    91.8 & 98.8 & 524.3 & 91.9 & Diverged \\ 
    140.7 & 162.7 & 1347.5 & 152.9 & \\
    \hline
    \end{tabular}
\end{center}
\vspace{-0.1in}
\caption{Language Model Performance (Perplexity) for the PTB \textbf{(Top)} and Text8 \textbf{(Bottom)} datasets.}
\label{fig:LM}
\end{figure}

\subsection{Accuracy and Speed of Estimation}
\label{sec:accuracy_speed}
For this experiment, we take a snapshot of the weights $\theta_y$ and the context vector $x$, after training the language model for a single epoch.  The number of examples in the snapshot is the mini-batch size $\times$ BPTT steps. i.e. (32 examples x 20 steps = 640 total) Using the snapshot, we show how well each approach estimates the partition function in Figure \ref{fig:Accuracy}. The x-axis is the number of samples used for the partition function estimate. The partition function is estimated with [50, 150, 400, 1000] samples for the PTB dataset and [50, 400, 1500, 5000] samples for the Text8 dataset. The accuracy of the partition function estimate is measured with the Mean Absolute Error (MAE). Tables \ref{fig:PTB_Time} and \ref{fig:Text8_Time} show the total computation time for estimating the partition function for all of the examples.

From Figure \ref{fig:Accuracy} and Tables \ref{fig:PTB_Time}, \ref{fig:Text8_Time}, we conclude the following:
\begin{itemize}
    \item Exact Gumbel is the most accurate estimate of the partition function with the lowest MAE for both datasets.
    \item MIPS Gumbel is 50\% faster than the Exact Gumbel but its accuracy is significantly worse.
    \item Exact Gumbel and LSH Gumbel are much slower than the Uniform IS and LSH approaches by several orders of magnitude.
    \item Our LSH estimate is more accurate than the Uniform IS and LSH Gumbel estimates.
    \item As the number of samples increases, the MAE for the Uniform IS and LSH estimate decreases.
\end{itemize}

\begin{figure} [ht]
\begin{center}
    \begin{tabular}{ |c|c|c|p{1.45cm}|p{1.45cm}| } 
    \hline
    Samples &	Uniform &	LSH     & Exact Gumbel &	MIPS Gumbel \\
    \hline
    50	    &	0.103	&	0.191	&	79.34	    &	45.72  \\
    150	    &	0.325	&	0.604	&	248.47	    &	140.91  \\
    400	    &	0.944	&	1.743	&	690.39	    &	406.11  \\
    1000	&	1.874	&	3.440	&	1,646.59    &   1,064.31 \\
    \hline
    \end{tabular}
\end{center}
\vspace{-0.15in}
\caption{Wall-Clock Time (seconds) for the Partition Function Estimate - PTB Dataset }
\label{fig:PTB_Time}
\end{figure}

\begin{figure} [ht]
\begin{center}
    \begin{tabular}{ |c|c|c|p{1.45cm}|p{1.45cm}| } 
    \hline
    Samples &	Uniform & LSH &	Exact Gumbel &	MIPS Gumbel     \\
    \hline
    50	    &	0.13	&	0.23	&	531.37	    &	260.75	\\
    400	    &	0.92	&	1.66	&	3,962.25	&	1,946.22	\\
    1500	&	3.41	&	6.14	&	1,4686.73	&	7,253.44	\\
    5000	&	9.69	&	17.40	&	42,034.58	&	20,668.61	\\
    \hline
    \end{tabular}
\end{center}
\vspace{-0.15in}
\caption{Wall-Clock Time (seconds) for the Partition Function Estimate - Text8 Dataset}
\label{fig:Text8_Time}
\end{figure}

\section{Discussion}
\label{sec:discussion}
In this section, we briefly discuss the implementation details. 

{\bf LSH Family}: For the experiments, we used the SimHash LSH family to estimate the partition function. [See \cite{shrivastava2014improved} for the MIPS formulation for the SimHash LSH Family] It is computationally efficient to generate the LSH fingerprints and the collision probability values $p$ for the LSH function. Generating the LSH fingerprints takes advantage of fast matrix multiplication operations while calculating the hash collision probability only requires normalizing the inner product between the weights $\theta_y$ and the context $x$. [See Figure \ref{fig:lsh_pipeline} and Section \ref{sec:SimHash}] 

{\bf Fixed sample set $S$ size}: It is often desirable to have a fixed-sized sample set $S$. However, the size of the sample set retrieved from the LSH data structure is stochastic and not directly controlled. Here is our approach for controlling the size of the sample set $S$. First, we tune the (K, L) parameters for the LSH data structure to retrieve a sample set $S$ whose size is close to the desired threshold. Then, we randomly sub-sample the sample set $S$ such that its size meets the desired threshold. The old probability is multiplied by the sampling probability to get the new values.

\begin{figure}[ht]
\begin{center}
\includegraphics[width=0.45\textwidth]{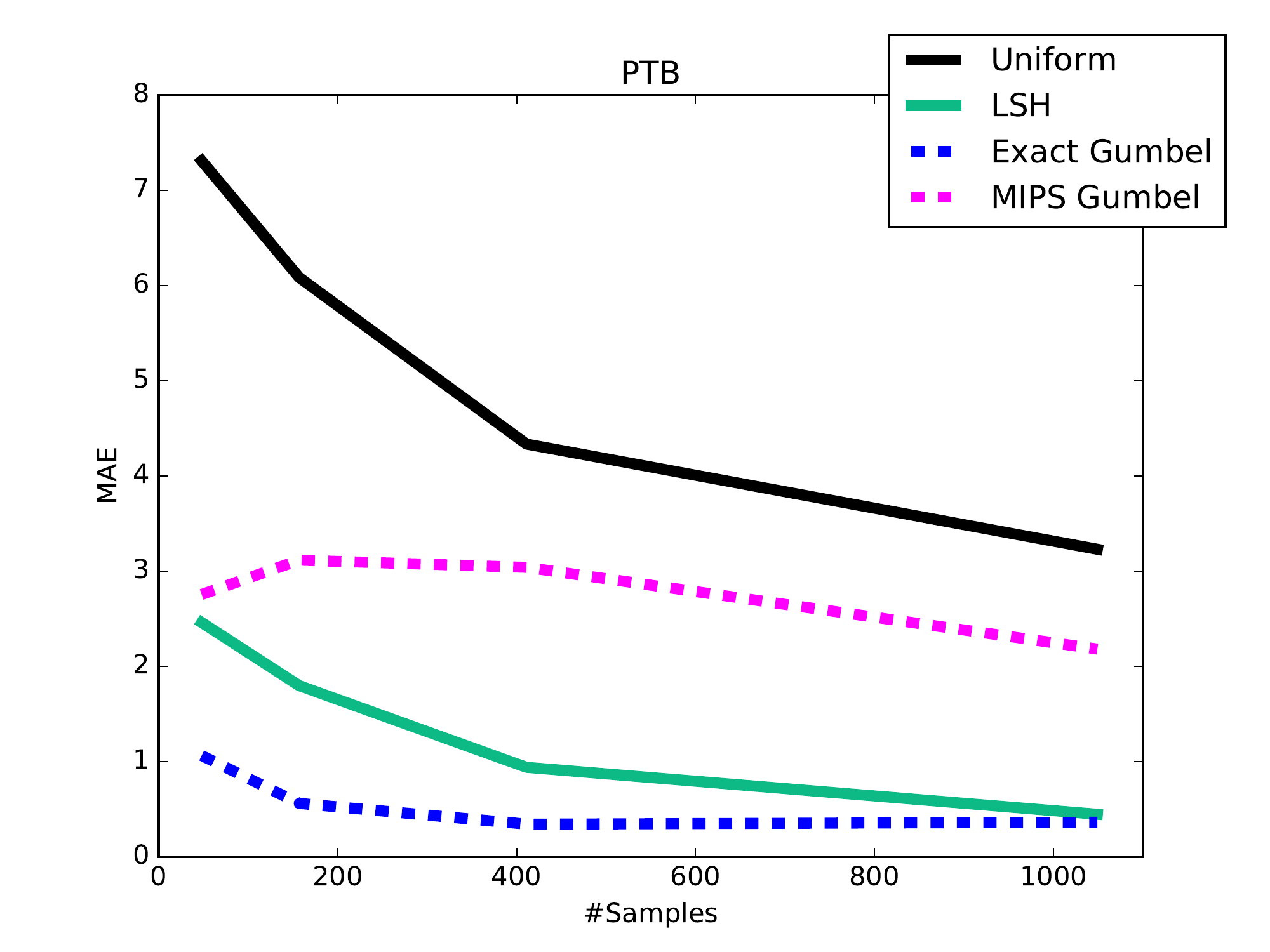}
\includegraphics[width=0.45\textwidth]{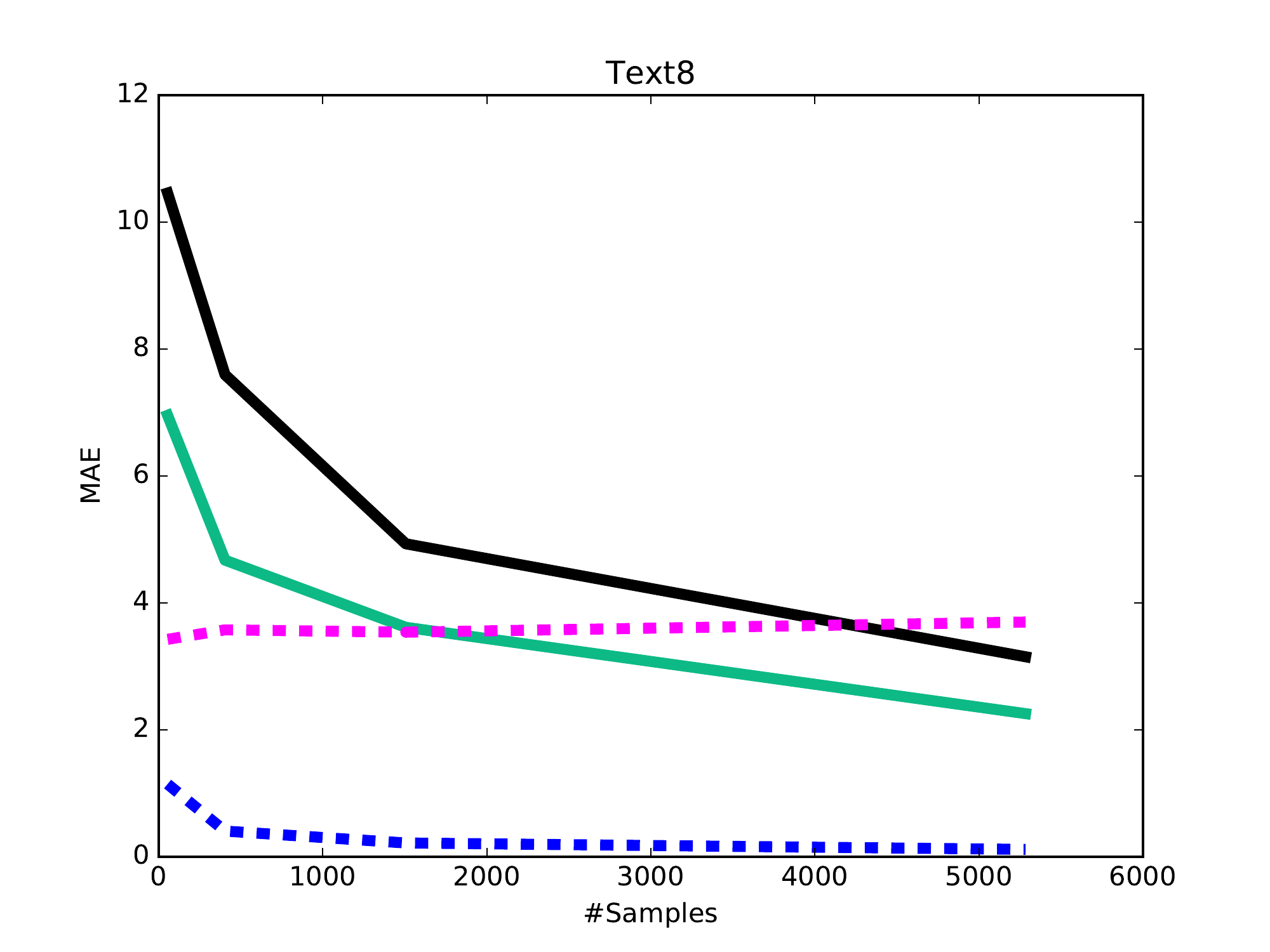}
\end{center}
\vspace{-0.25in}
\caption{Accuracy of the Partition Function estimate \\ Mean Absolute Error (MAE) - PTB \textbf{(Top)} and Text8 \textbf{(Bottom)}}
  \label{fig:Accuracy}
\end{figure}

\section{Acknowledgments}
The work of Ryan Spring was supported from NSF Award 1547433. This work was supported by Rice Faculty Initiative Award 2016.

\begin{small}
\bibliography{icml} 
\end{small}
\bibliographystyle{icml2017}

\newpage
\appendix
\section{Appendix}
\begin{theorem}
Assume there is a set of states $Y$. Each state $y$ occurs with probability $[p_1 \dots p_N]$.
For some feature vector $x$, function $f(y_i) = e^{\theta_{y_i} \cdot x}$
Then, there is a random variable whose expected value is the partition function.
$$ Est = \sum_{i=1}^{N} 1_{[y_i \in S]} \cdot \frac{f(y_i)}{p_i} $$
$$ \mathbb{E}[Est] = \sum_{i=1}^{N} f(y_i) = Z_\theta $$
\end{theorem}

\begin{theorem}
The variance of the partition function estimator is:
\begin{align}
Var[Est] &= \sum_{i=1}^{N} \frac{f(y_i)^2}{p_i} - \sum_{i=1}^{N} f(y_i)^2 \\
& + \sum_{i\ne j} \frac{f(y_i)f(y_j)}{p_ip_j} \mathrm{Cov}({\bf 1}_{[y_i \in S]} \cdot {\bf 1}_{[y_j \in S]})
\end{align}
If the states are selected independently, then we can write the variance as:
$$ Var[Est] = \sum_{i=1}^{N} \frac{f(y_i)^2}{p_i} - \sum_{i=1}^{N} f(y_i)^2 $$
\end{theorem}

\begin{proof}
The expression for the variance of the partition function estimator is: 
$$ Var[Est] = \mathbb{E}[Est^2] - \mathbb{E}[Est]^2 $$

\begin{align}
&Est^2 = \sum_{ij} 1_{[y_i \in S]} 1_{[y_j \in S]} \frac{f(y_i) f(y_j)}{p_i p_j} \\
&= \sum_{i} 1_{[y_i \in S]} \frac{f(y_i)^2}{p_i^2} + \sum_{i \neq j} 1_{[y_i \in S]} 1_{[y_j \in S]} \frac{f(y_i) f(y_j)}{p_i p_j}
\end{align}

{\bf Notice:}
$$\mathbb{E}[1_{[y_i \in S]}  1_{[y_j \in S]}] = \mathbb{E}[1_{[y_i \in S]}] \mathbb{E}[1_{[y_j \in S]}] + \mathrm{Cov}[1_{[y_i \in S]} 1_{[y_j \in S]}] $$
$$\mathbb{E}[1_{[y_i \in S]}] = p_i$$

\begin{align}
\mathbb{E}[Est^2] &= \sum_{i=1}^{N} \frac{f(y_i)^2}{p_i} + \sum_{i=1}^{N} f(y_i) [Z_\theta - f(y_i)] \\
&+ \sum_{i\ne j} \frac{f(y_i)f(y_j)}{p_ip_j} \mathrm{Cov}({\bf 1}_{[y_i \in S]} \cdot {\bf 1}_{[y_j \in S]}) \\
&= \sum_{i=1}^{N} \frac{f(y_i)^2}{p_i} +  Z_{\theta}^{2} - \sum_{i=1}^{N} f(y_i)^2 \\
&+ \sum_{i\ne j} \frac{f(y_i)f(y_j)}{p_ip_j} \mathrm{Cov}({\bf 1}_{[y_i \in S]} \cdot {\bf 1}_{[y_j \in S]})
\end{align}

{\bf Notice:}
$$\sum_{i=1}^{N} f(y_i)  = Z_{\theta}$$
$$\sum_{i \neq j}^{N} f(y_j) = \sum_{j=1}^{N} f(y_j) - f(y_i) = Z_{\theta} - f(y_i)$$

{\bf Notice:} $\mathbb{E}[Est]^2 = Z_{\theta}^{2}$.

Therefore,
\begin{align}
Var[Est] &= \sum_{i=1}^{N} \frac{f(y_i)^2}{p_i} - \sum_{i=1}^{N} f(y_i)^2 \\
& + \sum_{i\ne j} \frac{f(y_i)f(y_j)}{p_ip_j} \mathrm{Cov}({\bf 1}_{[y_i \in S]} \cdot {\bf 1}_{[y_j \in S]})
\end{align}

\end{proof} 

\begin{theorem}
  For any two states $y_1$ and $y_2$:
  $$ P(y_1|x;\theta) \ge P(y_2|x;\theta) \iff p_1 \ge p_2$$
  where 
  $$p_i = 1 - (1 - \mathcal{M}(\theta_{y_i} \cdot x)^K)^L$$
  $$P(y|x,\theta) \propto e^{\theta_y \cdot x}$$
\end{theorem}

\begin{proof}
Follows immediately from monotonicity of $e^{x}$ and $1 - (1 - \mathcal{M}(x)^K)^L$ with respect to the feature vector $x$. Thus, the target and the sample distributions have the same ranking for all the states under the probability. 
\end{proof}

\end{document}